\documentclass[twoside]{article}

\newcommand{\eqcontrib}{\textsuperscript{*}}

\usepackage[accepted]{aistats2026}

\usepackage{pifont}
\usepackage{lipsum}
\usepackage{graphicx}
\usepackage{algorithm}
\usepackage{algorithmic}
\usepackage{amsthm}
\usepackage{hyperref}
\hypersetup{ hidelinks }
\usepackage{amsmath}
\usepackage{amssymb}
\usepackage{xfrac}
\usepackage{mathtools}
\usepackage{tikz}
\usepackage{xcolor}
\usepackage{thm-restate}
\usepackage{enumitem}
\usetikzlibrary{positioning}
\usepackage{wrapfig}
\usepackage{natbib}

\definecolor{darkgreen}{rgb}{0.0, 0.5, 0.0}

\newcommand{\cO}{\mathcal{O}}
\newcommand{\cOO}[1]{\mathcal{O}\left( #1 \right)}

\newcommand{\Ceil}[1]{\lceil#1\rceil}

\renewcommand{\b}{\boldsymbol{\beta}}
\newcommand{\bt}{\b_{t}}

\newcommand{\cV}{\mathcal{V}}
\newcommand{\cB}{\mathcal{B}}

\newcommand{\OCO}{\textsc{OCO}}

\newcommand{\For}[2]{
  \FOR{#1}
  #2
  \ENDFOR
}

\usepackage{tikz}

\usepackage{stmaryrd}

\newcommand*\ie{{i.e.}} %

\newcommand{\bx}{\boldsymbol{x}}

\newcommand{\zeros}{\boldsymbol{0}}
\renewcommand{\hat}{\widehat}

\newcommand{\Log}[1]{\log\left(#1\right)}

\newcommand{\cA}{\mathcal{A}}
\newcommand{\cS}{\mathcal{S}}

\newcommand{\minOp}{\wedge}

\newcommand{\inner}[1]{\left\langle#1\right\rangle}
\newcommand{\Set}[1]{\left\{#1\right\}}

\newcommand{\tmm}{{t-1}}
\newcommand{\norm}[1]{\left\|#1\right\|}
\newcommand{\grad}{\nabla}
\newcommand{\abs}[1]{\left|#1\right|}
\newcommand{\sbrac}[1]{\left[#1\right]}

\newcommand{\brac}[1]{\left(#1\right)}
\newcommand{\sumtT}{\sum_{t=1}^{T}}
\newcommand{\cmp}{\boldsymbol{u}}
\newcommand{\w}{\boldsymbol{w}} %
\newcommand{\wt}{\w_{t}}

\newcommand{\bgt}{\boldsymbol{\ell}_{t}}
\newcommand{\EE}[1]{\mathbb{E}\sbrac{#1}}
\newcommand{\eg}{{e.g.}}

\newcommand{\Indicator}[1]{\mathbb{I}\Set{#1}}

\renewcommand{\v}{v}
\newcommand{\vt}{\v_{t}}
\newcommand{\cW}{\R^d}

\newcommand{\Ellhat}{\hat{\Ell}}
\newcommand{\Ell}{\boldsymbol{\ell}}

\newcommand{\g}{\ell}

\newcommand{\bghat}{\boldsymbol{\hat{\g}}}

\renewcommand{\tilde}{\widetilde}

\newcommand{\R}{\mathbb{R}}
\newcommand{\E}{\mathbb{E}}

\newcommand{\bs}[1]{\boldsymbol{#1}}

\usepackage[capitalize,noabbrev]{cleveref}

\begin{document}

\twocolumn[
\aistatstitle{Parameter-Free Dynamic Regret for Unconstrained Linear Bandits}

\newcommand{\unimi}{Universit\`{a} degli Studi di Milano}
\newcommand{\polimi}{Politecnico di Milano}
\newcommand{\intesa}{Intesa Sanpaolo Innovation Center}
\newcommand{\intesaai}{Intesa Sanpaolo AI Research}

\runningauthor{Alberto Rumi, Andrew Jacobsen, Nicol\`{o} Cesa-Bianchi, Fabio Vitale}

\aistatsauthor{ Alberto Rumi\eqcontrib\\\intesaai \And Andrew Jacobsen\eqcontrib\\\unimi,\\\polimi \AND Nicol\`{o} Cesa-Bianchi\\\unimi,\\\polimi\And Fabio Vitale\\ \intesa }

\aistatsaddress{}
]

\newcommand{\SecIntro}{Introduction}
\newcommand{\SecContributions}{Our Contributions}
\newcommand{\SecRelatedWorks}{Related Works}
\newcommand{\SecUniformSampling}{Combining Guarantees via Uniform Sampling}
\newcommand{\SecLinearBandits}{Linear Bandits}
\newcommand{\SecObliviousness}{On the Role of Obliviousness}
\newcommand{\SecImplications}{Relation to $S_T\sqrt{T}$ Lower Bounds}
\newcommand{\SecDiscussion}{Discussion}

\begin{abstract}
We study dynamic regret minimization in unconstrained adversarial linear bandit problems. In this setting, a learner must minimize the cumulative loss relative to an arbitrary sequence of comparators $\cmp_1,\ldots,\cmp_T$ in $\R^d$, but receives only \emph{point-evaluation feedback} on each round.
We provide a simple approach to combining the guarantees of several bandit algorithms, allowing us to optimally adapt to the number of switches $S_T = \sum_t\Indicator{\cmp_t \neq \cmp_\tmm}$ of an arbitrary comparator sequence.
In particular, we provide the \emph{first} algorithm for linear bandits achieving the optimal regret guarantee of order $\cO\big(\sqrt{d(1+S_T) T}\big)$ up to poly-logarithmic terms \emph{without prior knowledge of $S_T$}, thus resolving a long-standing open problem.
\end{abstract}

\section{\SecIntro}%
\label{sec:intro}

\newcommand{\ft}{f_{t}}

Online learning is a framework that models sequenctial decision-making against adversarial environments \citep{shalev2012online,orabona2019modern}. 
Formally, learning is modelled as a game played over a convex decision set $\mathcal{W}\subseteq\R^d$ for $T$ rounds of interaction. 
On each round $t\in[T]$, the learner chooses $\wt\in\mathcal{W}$, the adversary chooses chooses a loss function $\ft:\mathcal{W}\to\R$,
and the learner incurs loss $\ft(\wt)$. We focus in particular on the \emph{unconstrained linear bandit} setting, wherein $\mathcal{W}=\R^d$, the losses are linear functions $\ft(\w)=\inner{\Ell_t,\w}$ for some $\Ell_t\in\R^d$, 
and on each round the learner observes only the \emph{scalar} output of their decision, $\inner{\Ell_t,\wt}$. 

The classical way to evaluate the performance of online algorithms is through the notion of \textit{regret}, which compares the learner's expected cumulative loss to that of the best fixed strategy chosen in hindsight. While this can be a natural benchmark in some settings, it implicitly assumes that the environment is \emph{stationary}, \ie{}, that the optimal action does not change over time. In many practical scenarios, however, this can be unrealistic: user preferences, market conditions, or system dynamics can evolve unpredictably over time,
leading to situations where minimizing regret against a fixed comparator becomes meaningless. 
Instead, in this paper we focus on the more general notion of \emph{dynamic regret}, in which the learner's performance is evaluated relative to an arbirary
comparator \emph{sequence} $\cmp_1,\ldots,\cmp_T$ in $\R^d$:
\[
\EE{R_T(\cmp_{1},\ldots,\cmp_T)} = \EE{\sum_{t \in [T]} \inner{\Ell_t, \wt-\cmp_t}}\,,
\]
The standard notion of (static) regret is recovered when $\cmp_1 = \cdots = \cmp_T$. In this work, we consider a partially-oblivious model in which the comparator sequence is 
chosen obliviously at the start of the game, while the loss sequence is free to be chosen adaptively based on the learner's past decisions.

Notice that the ``complexity'' of the comparator sequence contributes to characterizing the difficulty of a given problem. If the comparator is allowed to change arbitrarily on each round, there is
no hope to achieve sublinear dynamic regret since the comparator sequence can ensure $\sumtT \ell_t(\cmp_t)=0$ even against completely unpredictable losses. On the other hand, we know that in the special case $\cmp_1=\cdots=\cmp_T$
we can ensure sublinear regret, since this is just the usual linear bandits setting. Thus, we are typically interested in algorithms which make dynamic regret guarantees that gracefully \emph{adapt} to some measure of complexity or \emph{variability} of the comparator sequence.
The classic way to quantify the variability of the  comparator sequence is via its \emph{path-length}:
\[
P_T = \sum_{t = 2}^T \norm{\cmp_{\tmm} - \cmp_{t}}\,,
\]
which offers a measure of cumulative drift of the comparator sequence.
A related measure of variability is the \emph{switching metric}, which more coarsely measures the number of times that the comparator changes over time:
$$
S_T = \sum_{t = 2}^T \mathbb{I}[\cmp_t \neq \cmp_{t-1}]~.
$$
While dynamic or switching regret provides a more realistic benchmark for non-stationary environments, it introduces significant challenges. Notably, achieving regret bounds that scale optimally with the comparator variability in bandit settings 
has previously only been achieved using preliminary knowledge of the comparator variability or a known upper bound on it \citep{agarwal2017corralling,marinov2021pareto,luo2022corralling}, which is prior knowledge which practitioners rarely have 
in practice.
Parameter-free methods which eliminate the need for prior knowledge of the comparator sequence are relatively well understood in the full-information setting \citep{zhang2018adaptive, cutkosky2020parameter,jacobsen2022parameter,jacobsen2023unconstrained,zhang2023unconstrained,jacobsen2024online}, but no prior works have successfully achieved parameter-free dynamic regret guarantees in the 
adversarial bandit setting.
Under bandit feedback, many of the parameter-free techniques developed for the full information setting are either not directly applicable or result in a suboptimal dependence on the time horizon and path length.

\subsection{\SecContributions}
We obtain optimal dynamic regret guarantees of order $\tilde\cO\big(\sqrt{d(1 + S_T)T}\big)$ against an adaptive adversary without prior knowledge of the number of changes $S_T$, 
resolving a long-standing open problem \citep{marinov2021pareto,luo2022corralling,auer2002nonstochastic}.
Key to obtaining this result is a technique for combining the guarantees of \emph{comparator-adaptive} base algorithms
inspired by a clever result in the full information setting \citep{cutkosky2019combining}, which we adapt to bandit feedback via a sampling trick. This simple trick enables us to easily combine the outputs of several bandit algorithms to achieve the best of their respective dynamic regret guarantees, effectively enabling us to ``tune'' hyperparameters on-the-fly. Such hyperparameter tuning arguments have been attempted by several prior works using sophisticated mixture-of-experts style arguments \citep{agarwal2017corralling,marinov2021pareto,luo2022corralling}, but have only achieved the optimal $\sqrt{S_T}$ dependence by leveraging \emph{a priori} knowledge of $S_T$. We expect that our approach might also find applications in other settings where bandit-over-bandit ensemble strategies have failed in the past.

\subsection{\SecRelatedWorks}

\textbf{Dynamic regret in online learning. }The study of dynamic regret was
initiated by \citet{HerbsterW98b,HerbsterW01} in the online regression and classification settings.
These results were first extended to the more general setting of online convex optimization (OCO) in the seminal work of \citet{zinkevich2003online}, where it was shown that online gradient descent achieves a bound of order $\cO\big((1+P_T)\sqrt{T}\big)$. \citet{yang2016tracking} improved the rate to $\cO(\sqrt{(1+P_T) T})$ by leveraging prior knowledge of $P_T$, and this bound was later shown to be minimax optimal by \citet{zhang2018adaptive}, who also provided the first algorithm achieving a matching upper bound without prior knowledge of $P_T$. 

Since then, several works have achieved generalizations of the $\cO(\sqrt{(1+P_T) T})$ guarantee by incorporating more adaptive, data-dependent quantities. These refinements typically replace the dependence on $T$ dependence with terms like the cumulative sum of the squared gradient norms $\sum_t\norm{\grad\ell_t(\wt)}^2$, or the temporal variation of the loss $\sum_t\sup_{\bx}\abs{\ell_t(\bx)-\ell_{t-1}(\bx)}$ \citep{cutkosky2020parameter,campolongo2021closer}. Significant effort in recent years has also been dedicated to also removing boundedness assumptions on the domain $\cW$ to adapt automatically to maximum comparator norm $\max_t\norm{\cmp_t}$ \citep{jacobsen2022parameter,jacobsen2023unconstrained,zhang2023unconstrained}. Various improvements in adaptivity can also be obtained under additional assumptions on the losses such as smoothness \citep{mokhtari2016online,zhao2020dynamic,jacobsen2024online,zhao2024adaptivity,jacobsen2025dynamic}.

\textbf{Comparator-adaptive methods. }Key to our results is the notion of \emph{comparator adaptive} online learning, where the goal is to design algorithms that adapt to the complexity of the comparator sequence (e.g., its maximum norm or a measure of its variability) without requiring prior knowledge about it. In the static comparator setting, this idea has been extensively studied under full-information feedback, where the optimal guarantee $R_T(\cmp)= O(\norm{\cmp}\sqrt{T})$, for any $\cmp\in \R^d$, can be obtained up to logarithmic terms \citep{mcmahan2012noregret,mcmahan2014unconstrained,orabona2016coin, cutkosky2018black}. Notably, \cite{cutkosky2019combining} showed that algorithms making comparator-adaptive guarantees can be easily combined, obtaining regret proportional to the best among them; this observation will be crucial to our approach in \Cref{sec:uniform}. In both linear and convex settings with bandit feedback, comparator-adaptive bounds were studied by \citet{van2020comparator} where they consider static regret and propose a black-box reduction approach, taking inspiration from the full-information reduction of \citet{cutkosky2018black}.

\textbf{Dynamic regret with bandit feedback. }In the bandit setting, the study of the closely-related notion of switching regret was initiated by \citet{auer2002using}, where a bound of $\cO(\sqrt{(1+S_T)T})$ was obtained using prior knowledge of the number of switches $S_T$.
The optimal regret bound for the non-stationary \emph{stochastic} bandit setting without \textit{a priori} variational knowledge was first obtained in \cite{auer2018adaptively}. %
Interestingly, \cite{marinov2021pareto} showed the impossibility of obtaining a $\cO(\sqrt{(1+S_T)T})$ bound against \textit{adaptive} adversaries in settings with finite policy classes, leaving the question open for unconstrained and oblivious settings.

More broadly, ensemble and meta-algorithm strategies have been proposed to adapt to unknown environment parameters in bandits. 
However, naive bandit-over-bandit schemes tend to incur significant overhead; for instance, running \textsc{EXP4} \citep{auer2002nonstochastic} as a meta-algorithm over \textsc{EXP3} base learners would yield $\cOO{T^{\sfrac{2}{3}}}$ regret due to the extra exploration needed \citep{odalric2011adaptive, cheung2019learning}. While these methods are flexible, they often pay a price in terms of worse regret bounds or higher variance, particularly in dynamic or tuning-free scenarios.
A notable example that avoids some of these issues is the \textsc{CORRAL} framework and its extensions \citep{agarwal2017corralling, luo2022corralling,marinov2021pareto}, which leverages a carefully-chosen sequence of regularizers at the meta-algorithm
level which is used to cancel out some of the additional variance at the base-algorithm level.
However, in the context of adapting to a dynamic comparator sequence, this approach 
still fails to obtain  $\sqrt{(1+S_T)T}$ regret unless $S_T$ is known \emph{a priori}, since the hyperparameters of the 
meta-algorithm itself need to be set with knowledge of $S_T$.

\paragraph{Notation.}\label{sec:assumptions}
Without further specifications, the notation $\norm{\cdot}$ refers to the Euclidean norm. The dual norm of $\w$ is denoted by $\norm{\w}_* = \sup_{\Ell : \norm{\Ell} \le 1} \inner{\Ell,\w}$. Given a norm $\norm{\cdot}$ and $\rho \ge 0$, we denote by $\cB_\rho := \Set{\bs{x} \,:\, \norm{\bs{x}} \le \rho}$ the closed ball of radius $\rho$, or the unit ball when the radius is not specified. 
The notation $\cO(\cdot)$ hides constant factors and $\tilde \cO(\cdot)$ hides constant and logarithmic factors.

\section{\SecUniformSampling}%
\label{sec:uniform}

Our approach is inspired by a framework for combining guarantees of  online learning algorithms proposed by \citet{cutkosky2019combining} for the simpler Online Covex Optimization (\OCO) setting. The framework applies to a class of algorithms which are \emph{adaptive} to the norm of an arbitrary comparator to obtain $\R_T(\cmp)\le \tilde O(\norm{\cmp}\sqrt{T})$, uniformly over all $\cmp\in\R^d$ \emph{simultaneously} \citep{mcmahan2013minimax,mcmahan2014unconstrained,orabona2016coin,cutkosky2018black}.  
One of the key properties that characterizes these comparator-adaptive guarantees is the fact the regret at the origin is constant, i.e., $R_T(\zeros)=\cO(1)$.

The key insight from \citet{cutkosky2019combining} is that if we have $N$
such \OCO\ algorithms $\cA_{1},\ldots,\cA_{N}$, we can achieve a regret competitive with the best guarantee among them,  $R_{T}(\cmp) = \cO(\min_{i}R_{T}^{\cA_{i}}(\cmp))$, by
simply
adding the iterates together. To see why, let $\wt^{(i)}$ denote the output of
$\cA_{i}$ on round $t$, and observe that if we play
$\wt=\sum_{i\in[N]}\wt^{(i)}$, then for any $j\in[N]$ we have
\begin{align*}
  R_{T}(\cmp)&:=\sumtT \inner{\bgt, \wt-\cmp} \\
  &= \sumtT \inner{\bgt, \wt^{(j)}-\cmp} + \sum_{i\ne j}\sumtT \inner{\bgt, \wt^{(i)}}\\
            &= R_{T}^{\cA_{j}}(\cmp)+\sum_{i\ne j}R_{T}^{\cA_{i}}(\zeros)
              \\
              &=\cO\left(R_{T}^{\cA_{j}}(\cmp)
   + N\right),
\end{align*}
where the last step holds because each algorithm guarantees $R_{T}^{\cA_{i}}(\zeros)=\cO(1)$. Moreover, since this holds for any $j\in[N]$, it must hold for the best among them.

While the iterate-adding approach is elegant in the full-information \OCO\ setting, it is not directly applicable under bandit feedback.
The fundamental issue is that when playing the sum
$\wt=\sum_{i\in[N]}\wt^{(i)}$, the feedback observed is only the total feedback 
$\langle\bgt, \sum_{i}\wt^{(i)}\rangle$, making it difficult to give proper feedback 
to any individual learner since the observed feedback includes 
includes contributions from all other learners' decisions.
Instead, we make the following observation: if on each round we sample one of the algorithms
uniformly at random and play \emph{only its action} on round $t$, then \emph{in expectation}, this is equivalent to playing $\sum_{i}\wt^{(i)}/N$. As a result, we
can still apply \emph{nearly} the same iterate-adding argument outlined above,
but we can now accurately assign feedback
since only one learner's action is played on each round.
The main difference from the iterate-adding approach is that
we need to rescale the comparator to account for the $1/N$ factor
that shows up in $\sum_{i}\wt^{(i)}/N$.

\begin{algorithm}[t]
\begin{algorithmic}
\STATE \textbf{Input:} Base algorithms $\big(\cA_n\big)_{n=1}^{N}$
\For{$t=1,\ldots,T$}{
\STATE Get $\wt^{(i)}$ from $\cA_i$ for all $i\in[N]$
  \STATE Sample $I_{t}\sim \text{Uniform}(N)$
  and play $\wt=\wt^{(I_{t})}$
  \STATE Receive feedback $\phi(\wt,\ell_t)$
  \STATE Send $\bghat_t^{(i)}=\phi(\wt,\ell_t)\Indicator{I_t=i}$ to $\cA_i$ for $i\in[N]$
}
\end{algorithmic}
\caption{Uniform Sampling Interface}\label{alg:uniform-interface}
\end{algorithm}
The procedure described above is summarized in \Cref{alg:uniform-interface} for a generic feedback oracle $\phi$, which returns a feedback signal $\phi(\wt,\ell_t)$ given the decision $\wt$. This captures first-order feedback when $\phi(\wt,\ell_t) = \bgt$ and bandit feedback when $\phi(\wt,\ell_t) = \ell_t(\wt)$. %
To build intuition, we begin with a warm-up result that analyzes the performance of this strategy in the first-order feedback setting.
\begin{restatable}{proposition}{UniformCombiningOLO}
\label{prop:uniform-combining-olo}
Let $\cA_{1},\ldots,\cA_{N}$ be online learning algorithms and let $\wt^{(i)}$
denote the output of $\cA_{i}$ on round $t$. Suppose that for
all $i$, $\cA_{i}$ guarantees
$R_{T}^{\cA_{i}}(\zeros)=\sumtT \ft(\wt^{(i)})-\ft(\zeros)\le G\epsilon$ for
any sequence of $G$-Lipschitz linear loss functions $f_1,\ldots,f_T$.
Then for any sequence $\cmp_{1},\ldots,\cmp_{T}$ in $\cW$,
\Cref{alg:uniform-interface} guarantees
    \begin{align*}
        &\EE{R_T(\cmp_{1},\ldots,\cmp_T)} \le\\
        &\qquad
        (N-1)\epsilon G + \min_{n\in[N]}\EE{R_T^{\cA_n}(N\cmp_1,\ldots,N\cmp_T)}
    \end{align*}
    where $R_T^{\cA_n}(N\cmp_1,\ldots,N\cmp_T)$ denotes the dynamic regret of $\cA_n$ played against losses $\bghat_t^{(n)}=\Ell_t\mathbb{I}\Set{I_t=n}$.
\end{restatable}
\begin{proof}
For all $t \ge 1$, let $\mathcal{F}_\tmm$ be the $\sigma$-algebra generated by the history up to the start of round $t$.
Let $\bghat_t^{(n)} = \Indicator{I_t = n}\bgt$, and note that $\wt = \sum_i\Indicator{I_t = i}\wt^{(i)}$.
Observe that for any $n \in [N]$ we can bound the
expected dynamic regret  by
\allowdisplaybreaks
\begin{align*}
    &\notag\E\Big[R_T(\cmp_{1},\ldots,\cmp_{T})\Big]\notag
= \EE{\sumtT \inner{\bgt, \wt-\cmp_t}}
\\&\quad=\E\Bigg[\sumtT \Big\langle\bgt, \sum_{i=1}^N \Indicator{I_t = i}\wt^{(i)}\Big\rangle\Bigg]
\\&\quad\notag \quad- \E\Bigg[\sumtT \inner{\bgt,\cmp_t}\underbrace{\E\Big[N\,\Indicator{I_t = n} \,\Big\vert\, \mathcal{F}_{t-1}\Big]}_{=1} \Bigg]
\\&\quad\notag= \EE{\sum_{i=1}^N \sumtT \inner{\bghat_t^{(i)}, \wt^{(i)}}} - \EE{\sumtT \inner{\bghat^{(n)}_t,N\cmp_t} }
\\&\quad\notag=
    \E \Bigg[ \sumtT \inner{\bghat_t^{(n)}, \wt^{(n)}-N\cmp_t}
+ \sum_{i\neq n}\underbrace{\sumtT \inner{\bghat_t^{(i)}, \wt^{(i)}}}_{=R_T^{\cA_i}(\bs{0})}\Bigg] \\
&\quad\notag\le
    \EE{ R_T^{\cA_n}(N\cmp_1,\ldots,N\cmp_T)} + (N-1)G\epsilon~. %
\end{align*}
The proof is concluded by observing $n$ was arbitrary, so we may choose the best $n\in[N]$.
\end{proof}

This theorem illustrates how uniform sampling acts as a rudimentary coordination mechanism for comparator-adaptive algorithms. Because each learner individually adapts to an arbitrary comparator, our combined strategy, on average, can be compared against the performance of any single algorithm in the collection. As we will see in the following sections, this mechanism enables us to ``tune'' hyperparameters on-the-fly, allowing us to adapt to unknown problem parameters such as the 
switching number $S_T$ without any prior knowledge.

It is important to note that, unlike the iterate-adding approach of \citet{cutkosky2019combining}, this method for combining guarantees can ends up increasing the comparator norm by a factor of $N$. Therefore, in the context of \OCO, the uniform sampling approach is a generally a worse option than the iterate-adding approach. Indeed, this trick is primarily of interest in settings where the
iterate-adding adding approach would complicate feedback, such as partial-feedback settings where the iterate adding approach makes individual loss estimation difficult.

\section{\SecLinearBandits}%
\label{sec:linear}
Recall from the previous section that the key property that we need to apply the uniform sampling strategy (\Cref{alg:uniform-interface}) is that the base algorithms $\cA_i$ guarantee a comparator-adaptive property, $R_T(\zeros)= \cO(1)$.
To obtain guarantees of this form in the bandit setting, we use a scale/direction decomposition introduced by \citet{van2020comparator}, which adapts 
a well-known reduction from the full-information setting to bandit feedback \citep{cutkosky2018black}.
The idea is to decompose the action $\wt$ into a \emph{scale} $v_t\in\R$
and a \emph{direction} $\bt\in\cB=\Set{\bs{x}\in\R^d: \norm{\bs{x}}\le 1}$, which will be learned by separate
online learning algorithms. In particular, observe that
for $M=\max_t\norm{\cmp_t}$ we have
\begin{align*}
&R_T(\cmp_1 \ldots \cmp_T) = \sumtT \inner{\Ell_t, \wt-\cmp_t}\\
&=\sumtT \inner{\Ell_t,\bt}v_t -\inner{\Ell_t,\cmp_t}\\
&=
\underbrace{\sumtT \inner{\Ell_t,\bt}(v_t -M)}_{=:R_T^{\cA_\cV}(M)} + M\underbrace{\sumtT  \inner{\Ell_t,\bt-\cmp_t/M}}_{=:R_T^{\cA_\cB}(\cmp_1/M,\ldots,\cmp_T/M)}\,
\end{align*}
where we've added and subtracted $M\inner{\Ell_t,\bt}$ in the last line.
Therefore, in order to ensure the required comparator-adaptive property $R_T(\zeros) = \cO(1)$ it suffices to provide a scale learner algorithm $\cA_\cV$ which guarantees $R_T^{\cA_\cV}(0)=\mathcal{O}(1)$ against the losses $v\mapsto \inner{\Ell_t,\bt}v$. This is a 1-dimensional \emph{static regret} problem for Online Linear Optimization with full information feedback, so we can apply any of the existing comparator-adaptive algorithms to ensure this property \citep{orabona2016coin,cutkosky2018black,mhammedi2020lipschitz,jacobsen2022parameter}. 
The following proposition then shows that 
when applied with the uniform sampling mechanism from the previous section, 
we are able to obtain a dynamic regret guarantee which matches that of the best 
algorithm in the set.

\begin{algorithm}[t!]
\begin{algorithmic}
\STATE \textbf{Input:} Scale learner $\cA_{\cV}$, direction learner $\cA_{\cB}$
\For{$t=1,\ldots,T$}{
  \STATE Get scale prediction $\vt$ from $\cA_{\cV}$ and direction prediction $\bt$ from
  $\cA_{\cB}$
  \STATE Play $\wt=\vt\bt$
  \STATE Observe $\inner{\wt,\Ell_{t}}$ and compute $g_t =
  \inner{\wt,\Ell_{t}}/\vt$
  \STATE Send $\inner{\Ell_t,\bt} = \inner{\Ell_t,\wt}/v_t$ to $\cA_\cB$ as the feedback on round $t$
  \STATE Send $v\mapsto \inner{\Ell_t, \bt} v$ to $\cA_{\cV}$ as the $t^\mathrm{th}$ loss
}
\end{algorithmic}
\caption{Scale / Direction Decomposition \citep{van2020comparator}}\label{alg:scale-dir-decomposition}
\end{algorithm}

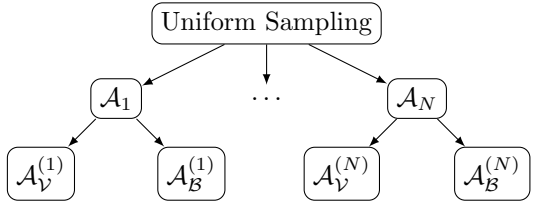
\begin{figure}[t!]
    \centering
\begin{tikzpicture}[
  every node/.style={draw, rounded corners, align=center, minimum height=1.0em},
  level 1/.style={sibling distance=20mm, level distance=10mm},
  level 2/.style={sibling distance=20mm, level distance=10mm},
  edge from parent/.style={draw, -latex}
]

\node {Uniform Sampling}
  child {node {$\cA_1$}
    child {node {$\cA_\cV^{(1)}$}}
    child {node {$\cA_\cB^{(1)}$}}
  }
  child {node[draw=none] {$\ldots$}}
  child {node {$\cA_N$}
    child {node {$\cA_\cV^{(N)}$}}
    child {node {$\cA_\cB^{(N)}$}}
  };

\end{tikzpicture}
    \caption{Illustration of how the Uniform Sampling interface interacts with each base algorithm $\cA_i$. Each base algorithm internally applies the direction and scale decomposition, using its own hyperparameters.}
    \label{fig:schema}
\end{figure}

\begin{minipage}{\columnwidth}
\begin{restatable}{proposition}{UniformCombining}\label{prop:uniform-combining-unconstrained}
Pick $N$ base algorithms implementing \Cref{alg:scale-dir-decomposition}, where each scale learner $\cA_\cV^{(n)}$ is over $\R$ and each direction learner $\cA_\cB^{(n)}$ is over $\cB=\Set{\bs{x}\in\R^d: \norm{\bs{x}}\le 1}$.
Suppose that for any sequence of $G$-Lipschitz linear losses $f_1,\ldots,f_T$ over $\R$ and for any $\epsilon > 0$, the regret of $\cA_{\cV}^{(n)}$ satisfies $R_{T}^{\cA_{\cV}^{(n)}}(0)=\sumtT \ft(\vt^{(n)})-\ft(0)\le G\epsilon~$.
Then, for any $n\in[N]$ and any sequence $\cmp_{1},\ldots,\cmp_{T}$ in $\R^{d}$, \Cref{alg:uniform-interface} guarantees
  \begin{align*}
    \EE{R_{T}(\cmp_{1},\ldots,\cmp_T)} 
    &\le
       (N-1)G\epsilon
       +
       \EE{R_{T}^{\cA_{\cV}^{(n)}}\!\!(MN)}\\
       &\quad
       +MN\EE{ R_{T}^{\cA_{\cB}^{(n)}}\!\!\brac{\frac{\cmp_{1}}{M},\ldots,\frac{\cmp_{T}}{M}}},
  \end{align*}
  where  $M=\max_{t}\norm{\cmp_{t}}$.
\end{restatable}
\end{minipage}
\begin{proof}
The result is immediate by applying
\Cref{prop:uniform-combining-olo}
followed by
the scale / direction decomposition.
For any $n\in[N]$, we have
\allowdisplaybreaks
\begin{align*}
    &\EE{R_{T}(\cmp_{1},\ldots,\cmp_{T})}
    =
      \EE{\sumtT \inner{\Ell_{t},\wt-\cmp_{t}}}\\
    &\le
       \EE{\sumtT \inner{\Ellhat_{t}^{(n)},\bt^{(n)}v_t^{(n)}-N\cmp_{t}}}+(N-1)G\epsilon\\
      &=
      \E\bigg[\sumtT \inner{\Ellhat_{t}^{(n)},\bt^{(n)}}v_{t}^{(n)}-\inner{\Ellhat_{t}^{(n)}, \bt^{(n)}}NM\bigg]\\
      &\qquad  
      +\E\bigg[\sumtT \inner{\Ellhat_{t}^{(n)}, \bt^{(n)}}NM-\inner{\Ellhat_{t}^{(n)}, N\cmp_{t}}\bigg]\\
      &\qquad
      +(N-1)G\epsilon\\
    &=
   (N-1)G\epsilon+ \EE{R_{T}^{\cA_{\cV}^{(n)}}\!\!\!(MN)}\\&\qquad+MN \EE{R_{T}^{\cA_{\cB}^{(n)}}\!\!\!\brac{\frac{\cmp_{1}}{M},\ldots,\frac{\cmp_{T}}{M}}}~.\qedhere
  \end{align*}
\end{proof}

As observed above, the key insight of \cite{van2020comparator} is that the feedback received by the scale learner is actually full-information feedback; indeed, a scale learner's
loss function
$v\mapsto\inner{\Ell_t,\bt} v$ can be precisely recovered from $\inner{\Ell_t,\wt}=\inner{\Ell_t,\bt}v_t$ by dividing it by the scale $v_t$,
so no tricky loss estimation is required for the scale learner.
Moreover, the scale learner faces a \emph{static} regret problem, so overall to meet the condition of \Cref{prop:uniform-combining-unconstrained},
it will suffice to apply any comparator-adaptive \OCO\ algorithm for static regret as the scale learner, leading to
a guarantee of the form $R_T^{\cA_\cV^{(n)}}(NM)=\widetilde{\mathcal{O}}(NM\sqrt{T})$. In what follows we choose
\citet[Algorithm 4]{jacobsen2022parameter} for concreteness, but really any algorithm guaranteeing $R_T(0)=\cO(1)$ would suffice.

All that remains is to provide a
bandit algorithm which can guarantee the optimal $\sqrt{dS_T T}$ switching regret on the unit-ball when tuned optimally. 
For this, we can apply \citet[Algorithm 2]{luo2022corralling}, which guarantees
for any sequence $\tilde\cmp_1,\ldots,\tilde\cmp_T$ in $\cB=\Set{\bs{x}\in\R^d:\norm{\bs{x}}\le 1}$ that
\begin{align*}
    \EE{R_T^{\cA_\cB}(\tilde\cmp_1,\ldots,\tilde\cmp_T)}=\tilde \cO\brac{\sqrt{d(1+ S_T)T}}, 
\end{align*}
so long as the hyperparameters are set optimally in terms of $S_T=\sum_{t=2}^T\mathbb{I}\Set{\tilde\cmp_t-\tilde\cmp_\tmm}$, $T$, and $d$. 
In particular, their approach combines instances of linear bandit algorithms 
using a carefully-designed multi-scale meta-algorithm, and obtaining the 
desired $\tilde\cO(\sqrt{d(1+S_T) T})$ bound requires that the hyperparameters \emph{of the 
meta-algorithm} be set using knowledge of $S_T$, which is of course unknown in practice. However, 
using our \Cref{prop:uniform-combining-unconstrained} we can effectively 
\emph{guess} $S_T$ up to a constant factor by running an 
instance of their algorithm parameterized using $\hat S_n$ for each $\hat S_n$ in a geometrically-spaced grid of candidates
\begin{align}\cS = \Set{2^n\minOp 2T, n=0,1,\ldots}\cup\Set{0}.\label{eq:grid-linear-bandits}\end{align}
Note that even when $\hat S_i$ is misspecified, the scale/direction decomposition above still ensures that the overall algorithm 
has the comparator-adaptive property $R_T(\zeros)=\cO(1)$ due to the  multiplicative dependence between the regret of the unit-ball learner and $M=\max_t\norm{\cmp_t}$
in \Cref{prop:uniform-combining-unconstrained}. Moreover, since $S_T\in[0,T-1]$ for any sequence of comparators, we know that there is always 
an $\hat S_n\in\cS$ satisfying $\hat S_n\le S_T\le 2\hat S_n$ while also maintaining $N=\cO(\Log{T})$.
Therefore, we are able to leverage these facts to arrive at the following parameter-free dynamic regret guarantee.

\begin{restatable}{theorem}{OptimalSwitchingContinuousArm}\label{thm:optimal-switching-continuous-arm}
  Let $\cS$ be defined as in \Cref{eq:grid-linear-bandits}.
  For all $i\in \Set{0,\ldots,\abs{\cS}}$, let $\cA_i$ be an instance of \Cref{alg:scale-dir-decomposition} such that the direction learner $\cA_{\cB}^{(i)}$ 
  be characterized by \citet[Theorem 1]{luo2022corralling} with $S=\hat S_i\in \cS$, and 
  let $\cA_{\cV}^{(i)}$ be an instance of \citet[Algorithm 4]{jacobsen2022parameter}.
  Then for any sequence $\cmp_{1},\ldots,\cmp_{T}$ in $\R^d$,
  \Cref{alg:uniform-interface} guarantees
  \begin{align*}
    \EE{R_{T}(\cmp_{1},\ldots,\cmp_{T})}
    &=
      \tilde\cO\brac{MG\sqrt{N(1+S_T)T}}.
  \end{align*}
  where $S_T=\sum_{t=2}^T\mathbb{I}\Set{\cmp_t\ne\cmp_\tmm}$, $M=\max_t\norm{\cmp_t}$, and $N=\lceil\log_2(T)+1\rceil$.
\end{restatable}
\begin{proof}
  By \Cref{prop:uniform-combining-unconstrained}, for any
  $n\in[\abs{\cS}]$, we have
  \begin{align*}
    \EE{R_{T}(\cmp_{1},\ldots,\cmp_{T})}
    &\le
      \E\Big[ (N-1)G\epsilon+R_{T}^{\cA_{\cV}^{(i)}}(MN) \\
      &\quad+ MN R_{T}^{\cA_\cB^{(i)}}\brac{\frac{\cmp_{1}}{M},\ldots,\frac{\cmp_{T}}{M}}\Big]~.
  \end{align*}
  where $M=\max_{t}\norm{\cmp_{t}}$.
  The regret guarantee \citep[Theorem 1]{jacobsen2022parameter} of any of the $\cA_{\cV}^{(i)}$  yields
  \begin{align*}
    \EE{R_{T}^{\cA_{\cV}^{(i)}}(NM)}
    &=\tilde O\brac{ MN\E\sbrac{\sqrt{\sum_{t}\inner{\hat \Ell_{t}^{(i)},\bt^{(i_{t})}}^{2}}}}\\
    &=
      \tilde O\brac{ MG\sqrt{NT}},
  \end{align*}
  where we've used obliviousness of the comparator sequence and the facts that
  $\norm{\bt^{(i_{t})}}\le 1$ for all $t$ and $\bghat_t^{(i)}=\Ell_t\mathbb{I}\Set{I_t=i}$, so that
  $\EE{\sqrt{\sumtT \|\hat\Ell_{t}^{(i)}\|_{*}^{2}}}\le \sqrt{\sumtT \EE{\norm{\Ell_{t}}^{2}_{*}\Indicator{i_{t}=i}}}\le G\sqrt{T/N}$
  via Jensen's inequality.
  Moreover, letting $\tilde\cmp_{t}=\frac{\cmp_{t}}{M}$, there is a $\hat S_i\in \cS$ such that 
  $\hat S_i\le S_T\le 2\hat S_i$, so applying \citet[Algorithm 2]{luo2022corralling} 
  with hyperparameters set according to \citet[Theorem 1]{luo2022corralling} with $S=\hat S_i$
  yields
  \begin{align*}
    \EE{R_{T}^{\cA_{\cB}^{(i)}}(\tilde \cmp_{1},\ldots,\tilde\cmp_{T})}
    &=
    \tilde \cO\brac{G\sqrt{d (1+S_T) T}},
  \end{align*}
  so, combining with the previous display and the regret decomposition from \Cref{prop:uniform-combining-unconstrained} 
  \begin{align*}
    \EE{R_{T}(\cmp_{1},\ldots,\cmp_{T})}
    &=
    \tilde O\Bigg(N\epsilon G + GM\sqrt{NT}\\
    &\qquad+MN G\sqrt{d(1+S_T)T}\Bigg)\\
    &=
      \tilde O\brac{GM\sqrt{d(1+S_T) T}}~,
  \end{align*}
    where we've observed that $N=\Ceil{\log(T)+1}$ and hidden poly-logarithmic factors.
\end{proof}

We note that despite the simplicity of the argument, this result is in fact the first $\tilde \cO(\sqrt{(1+S_T) T})$ dynamic regret bound 
obtained without prior knowledge of $S_T$ in non-stochastic linear bandit settings. One might worry that the result above 
might contradict existing impossibility results for $\sqrt{(1+S_T) T}$  dynamic regret against adaptive adversaries.
Indeed, \citet{marinov2021pareto} show that
$S_T\sqrt{T}$ dynamic regret can be forced under very general conditions 
when bandit feedback is involved. 
Our result 
is avoiding this lower bound in two ways. First, the construction of \citet{marinov2021pareto}
uses \emph{both} an adaptive comparator and an adaptive loss sequence, while 
our approach explicitly leverages obliviousness of the comparator sequence.
Second, the lower bound construction 
in that work is for finite policy classes, which does not capture the unconstrained linear bandit problem. Intuitively, the reason why this distinction is important is that in finite policy class settings one can construct instances where choosing the comparator action is strictly better than any other action, leading to a positive gap between the loss of the learner and comparator. As such, lower bounds can be constructed by quantifying how
much exploration is needed to identify the optimal action, since any rounds of exploration will lead to a non-negative addition to the total regret.
However, in an unconstrained setting, such gaps generally do not hold uniformly, 
and it is difficult in general to bound how much negative regret the learner may accumulate by making risky bets. For instance, it is possible for the learner to generate arbitrarily large negative regret on any round in which $\wt$ is in the same half-plane as $-\Ell_t$ by taking $\norm{\wt}$ to be large.

\subsection{\SecObliviousness}%
\label{sec:oblivious}

While our approach in the previous section successfully achieves a parameter-free $\tilde \cO(\sqrt{S_T T})$ dynamic regret bound
against an adaptive loss sequence, the argument crucially relies on obliviousness \emph{of the comparator sequence}.
Most notably, \Cref{prop:uniform-combining-unconstrained} decomposes the 
expected dynamic regret in terms of the expected dynamic regret of an algorithm 
applied on the unit-ball, $M\E[R_T^{\cA_\cB}(\cmp_1/M,\ldots,\cmp_T/M)]$, but this
required moving the expectation passed $M=\max_t\norm{\cmp_t}$. For this, $M$ should at least be chosen 
obliviously, particularly so when we only have an \emph{expected} regret guarantee for the direction learner. 

More generally, the essential difficulty is that 
any time one hopes to leverage a 
bound on
$\E[\|\Ellhat_t\|_*^2]$, a comparator-adaptive guarantee
will require obliviousness of the comparator sequence 
to bypass the $M\sqrt{1+S_T}$ part of the bound.
Indeed, suppose we could directly reduce the problem to OLO 
on the loss estimates $\Ellhat_t$ to obtain
\begin{align*}
    \EE{R_T(\cmp_{1},\ldots,\cmp_T)}\lessapprox\EE{M\sqrt{(1+S_T)\sumtT \|\Ellhat_t\|^2_*}}.
\end{align*}
From here, we might hope to apply Jensen's inequality to arrive at
\begin{align*}
    \EE{R_T(\cmp_{1},\ldots,\cmp_T)}\lessapprox M\sqrt{(1+S_T)\sumtT \EE{\|\Ellhat_t\|_*^2}},
\end{align*}
but this is only valid if $M\sqrt{1+S_T}$ is deterministic. This issue could potentially 
be overcome by leveraging a loss estimate satisfying an \emph{almost-sure bound} such as $\|\Ellhat_t\|=\tilde \cO(d\norm{\Ell_t})$,
which would obviate the need for an oblivious comparator assumption entirely.
We conjecture that this is indeed possible in unconstrained settings, though requires a different construction 
than our approach in the previous section.
We leave this as an important  direction for follow-up work.

Alternatively, one might hope to avoid these considerations by  applying Cauchy-Schwarz inequality to get
$\sqrt{\EE{M^2(1+S_T)}\E[\sumtT \|\Ellhat_t\|^2_*}]$. However, this bound 
fails to be meaningful in general because it 
decouples the comparator from the loss sequence.
Indeed, even with perfect loss estimates $\Ellhat_t=\Ell_t$ it is easy to define events such that the decoupled bound 
becomes vacuous even for $S_T=0$ while 
$\sqrt{\EE{M^2(1+S_T)\sumtT \|\Ellhat_t\|^2_*}}$ remains sublinear.

\section{\SecDiscussion}

In this paper, we propose a remarkably simple approach to tuning the hyperparameters of unconstrained bandit algorithms on-the-fly, enabling adaptivity to problem parameters that cannot be directly observed or estimated, such as the switching number of the comparator sequence. This leads to the first optimal parameter-free dynamic regret bound for unconstrained linear bandits in adversarial settings. We expect that this technique can be applied in other contexts involving hard-to-estimate quantities, though it it crucially relies on the scale/direction decomposition discussed in \Cref{sec:linear}. Extending this approach to more general constrained action sets remains non-trivial, as irregular geometries can significantly complicate the decoupling of scale and direction.

A natural question is whether it is possible to recover guarantees that scale with the standard path length $P_T$ in linear bandits, instead of the more coarse switching guarantees obtained in this work. One possible approach might be to directly optimize in the action space---\eg{}, via mirror descent—and then construct a distribution $p_t$ such that $\E_{p_t}[\w] = \wt$, playing $\w \sim p_t$. However, the variance control of such sampling methods remain challenging  in the unconstrained setting.
Moreover, as mentioned in \Cref{sec:oblivious}, we suspect that it is possible in the unconstrained linear bandit setting to remove the oblivious comparator assumption as well, resulting in parameter-free dynamic regret guarantees against a
fully-adaptive adversary, though we leave this as a direction for future investigation.

Lastly, to obtain the desired $\cO(\sqrt{d(1+S_T) T})$ bound, we had to resort to running a collection of $\cO(\log(T))$ instances of the multi-scale corralling algorithm of \citet{luo2022corralling}, which itself ends up requiring $\cO(T\log T)$ per-round computation. We leave as an open question whether the optimal $S_T$ dependence can still be obtained using the usual $\cO(d\log T)$ per-round computation obtained in the full-information setting.

\subsubsection*{Acknowledgements}
NCB and AJ acknowledge the financial support from the EU Horizon CL4-2022-HUMAN-02 research and innovation action under grant agreement 101120237, project ELIAS (European Lighthouse of AI for Sustainability).

\bibliography{biblio}
\bibliographystyle{icml-2026-nourl}

\end{document}